\documentclass{article}

\usepackage[preprint,nonatbib]{neurips_2021}

\usepackage[utf8]{inputenc} 
\usepackage[T1]{fontenc}    
\usepackage{hyperref}       
\usepackage{url}            
\usepackage{booktabs}       
\usepackage{amsfonts}       
\usepackage{nicefrac}       
\usepackage{microtype}      
\usepackage{xcolor}         

\usepackage{amsmath}
\usepackage{amssymb}
\usepackage{amsthm}
\usepackage{graphicx}

\newtheorem{theorem}{Theorem}
\newtheorem{lemma}{Lemma}

\newtheorem{corollary}{Corollary}

\DeclareUnicodeCharacter{FFFD}{ }

\title{On the Expressive Power of Self-Attention Matrices}

\author{%
  Valerii Likhosherstov
  \thanks{equal contribution}\\
  Google Brain \& University of Cambridge\\
  \texttt{vlikhosherstov@google.com}\\
   \And
  Krzysztof Choromanski
  \footnotemark[1]\\
  Google Brain \\
   \And
  Adrian Weller\\
  University of Cambridge \& The Alan Turing Institute\\
}

\begin{document}

\maketitle

\begin{abstract}
    Transformer networks are able to capture patterns in data coming from many domains (text, images, videos, proteins, etc.) with little or no change to architecture components. We perform a theoretical analysis of the core component responsible for signal propagation between elements, i.e. the self-attention matrix. In practice, this matrix typically exhibits two properties: (1) it is sparse, meaning that each token only attends to a small subset of other tokens; and (2) it changes dynamically depending on the input to the module. With these considerations in mind, we ask the following question: \textbf{Can a fixed self-attention module approximate arbitrary sparse patterns depending on the input? How small is the hidden size $\boldsymbol{d}$ required for such approximation?} We make progress in answering this question and show that the self-attention matrix can provably approximate sparse matrices, where sparsity is in terms of a bounded number of nonzero elements in each row and column. While the parameters of self-attention are fixed, various sparse matrices can be approximated by only modifying the inputs. Our proof is based on the random projection technique and uses the seminal Johnson-Lindenstrauss lemma. Our proof is constructive, enabling us to propose an algorithm for finding adaptive inputs and fixed self-attention parameters in order to approximate a given matrix. In particular, we show that, in order to approximate any sparse matrix up to a given precision defined in terms of preserving matrix element ratios, \textbf{$\boldsymbol{d}$ grows only logarithmically with the sequence length} $\boldsymbol{L}$ (i.e. $\boldsymbol{d = O(\log L)}$).
\end{abstract}

\section{Introduction}

Transformer networks have demonstrated strong performance 
in the area of large-scale deep learning, coming close to or beating the state of the art 
in a wide range of tasks. 
Initially proposed in the context of neural machine translation \cite{transformer}, Transformers were found to generalize well across a variety of natural language processing tasks when pretrained on large text corpora \cite{bert,gpt-2,gpt-3}. These successes facilitated the application of Transformers in other domains. For instance, in biology, Transformers pretrained on large corpora of proteins were shown to predict proteins' structure and function \cite{endtoend,strfunc}, and to generate protein sequences with specific properties \cite{progen}. Another exciting advancement was the emergence of Vision Transformers \cite{vit} and, later, Video Vision Transformers \cite{vivit}. Thus,  Transformers appear domain-agnostic and can learn any priors once a suitable large-scale dataset is provided. Finally, Transformers were recently shown to be applicable for end-to-end training on large-scale multimodal data of images with textual annotations extracted from the Internet \cite{clip,align}. The resulting models are highly generalizable and perform very well in zero-shot classification from scratch, and when fine-tuned on standard benchmarks of a smaller scale.

The omnivorous nature of these models suggests that Transformers and their core component, self-attention, have an inherent ability to capture useful patterns 
in the data regardless of the domain. A thorough 
analysis is required to gain a deeper understanding of this remarkable phenomenon. We take a step in this direction by analyzing the expressiveness of the self-attention module.

In self-attention, dependencies between elements of the input are propagated via a self-attention matrix, which can be thought of as an input-dependent linear projection applied to the input. By the definition, this right stochastic matrix (i.e. having nonnegative elements with rows summing up to $1$) encodes input-dependent patterns in the data. Therefore, we aim to analyze the expressiveness of this matrix, to understand how flexible these input-dependent patterns can be. Importantly, we consider the setup when \textbf{the hidden dimension of self-attention $\boldsymbol{d}$ is much smaller than the sequence length $\boldsymbol{L}$} aiming to characterize relationships between the two. Small $d$ is important in practice, because it facilititates computational efficiency of Transformers, which are notorious for their high compute demand and $\mathrm{CO}_{2}$ footprint \cite{co2}.

\begin{figure}[t!]
    \centering
    \includegraphics[width=0.9\textwidth]{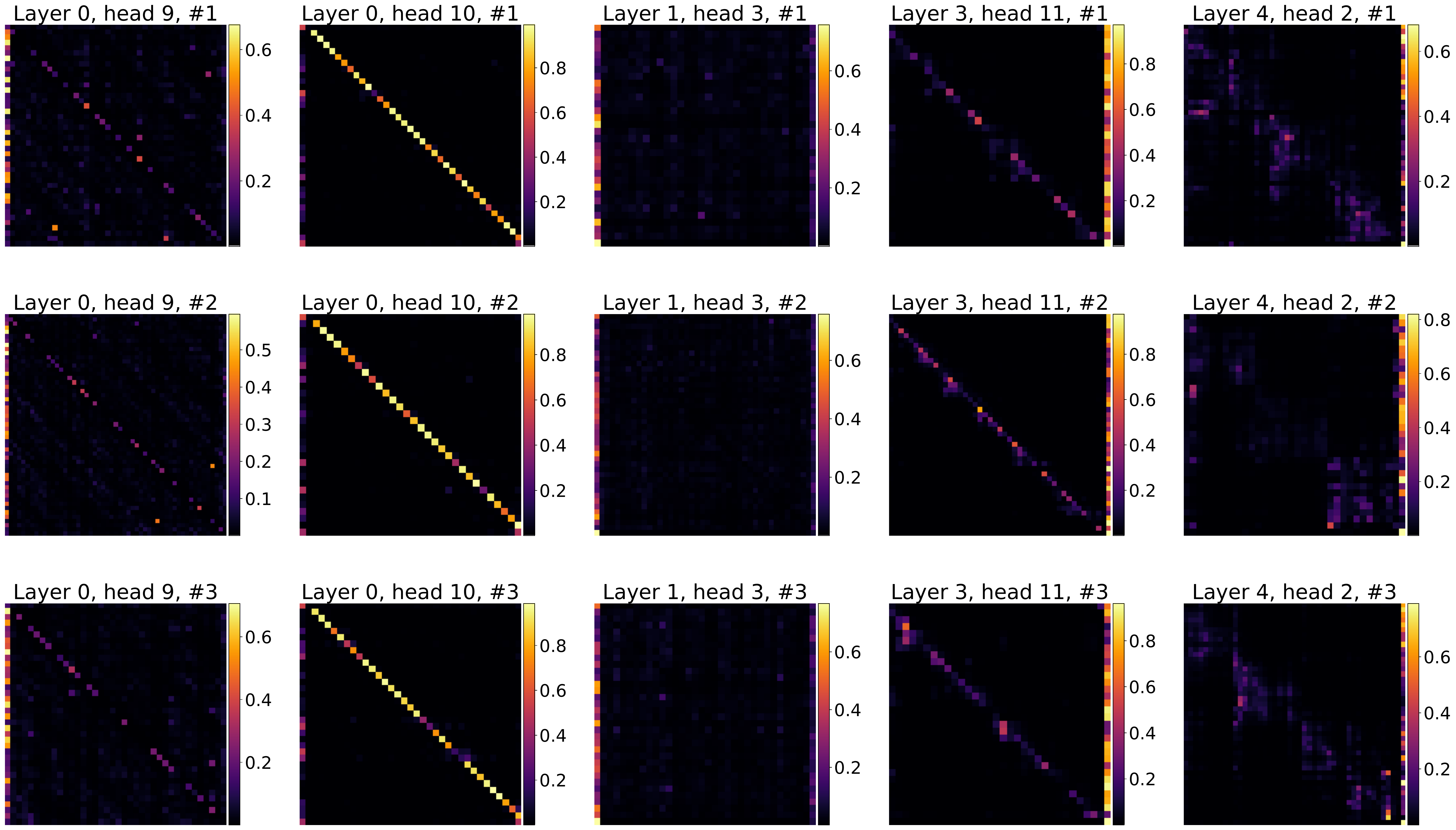}
    \caption{Examples of self-attention matrices appearing in the trained DistilBERT model \cite{distilbert}. Each column corresponds to a randomly chosen self-attention module defined by layer and attention head in the model. Cells in each column correspond to realisations of self-attention matrix for randomly sampled input sentences from the text corpus. We observe that self-attention matrices are 1) sparse and 2) the sparsity pattern depends on the input.}
    \label{fig:distilbert}
\end{figure}

While it is hard to tackle the most general setup of approximating \textit{any} possible right stochastic matrix with self-attention when $d \ll L$ (the case when $d \geq L$ was considered in \cite{lrbottleneck}), we narrow down the scope by making the following reasonable assumptions, often observed in practice:
\begin{enumerate}
    \item Self-attention matrices are approximately \textit{sparse}, meaning that most of the elements of the matrix are near zero. In other words, each token of the output only depends on a small number of input tokens. For instance, in neural machine translation, output words usually depend on a short context near the word they translate.
    \item Self-attention matrices are \textit{dynamic}, meaning that, with the fixed trained weights, sparsity patterns change depending on the input to the module. In our analysis we should, therefore, consider flexibility of the self-attention matrices assuming that weights are fixed.
\end{enumerate}

These assumptions typically hold in practice, as illustrated in Figure \ref{fig:distilbert}. The dynamic sparsity assumption, in particular, has provided insight for a series of results related to fast computation of self-attention for long sequences \cite{reformer,clustered,routing}. Consequently, two questions of interest are: \textbf{Can a fixed self-attention module approximate arbitrary sparse patterns depending on the input? How small is the latent dimension $\boldsymbol{d}$ required for such approximation?}

We make progress in addressing these questions by theoretically showing that there exist self-attention weights such that, when the precision of approximation is fixed, \textbf{$d$ grows only logarithmically with the sequence length $L$} to approximate any sparse matrix by only changing the input to the module. Here, the approximation precision is defined in terms of preserving attention weight ratios and sparsity is characterized by the bounded number of nonzero elements in each row and column. Our proof method uses random projection techniques and the seminal Johnson-Lindenstrauss lemma.

We commence by defining the self-attention module and self-attention matrix. After that, we formulate the main theoretical result of the paper and proceed with the proof. Finally, we present experimental simulations, discuss related work and make concluding remarks.


\section{Prerequisites: self-attention module}

Let $L$ be the length of a processed sequence and $d_{hid}$ be the size of a hidden representation passed through the neural network. We define the \textit{unnormalized self-attention matrix} as a parametrized mapping from the current hidden state $X \in \mathbb{R}^{L \times d_{hid}}$ into $\mathbb{R}^{L \times L}$. The mapping depends on two learnable parameter matrices $W_\mathcal{Q}, W_\mathcal{K} \in \mathbb{R}^{d_{hid} \times d}$, $d \leq d_{hid}$, and is defined as
\begin{equation} \label{eq:usam}
    \mathrm{USAM} (X; d, W_\mathcal{Q}, W_\mathcal{K}) = \exp (X W_\mathcal{Q} W_\mathcal{K}^\top X^\top),
\end{equation}
where $\exp (\cdot)$ is an elementwise exponent. Next, we define the \textit{(normalized) self-attention matrix} as
\begin{equation} \label{eq:sam}
    \mathrm{SAM} (X; d, W_\mathcal{Q}, W_\mathcal{K}) = \mathrm{diag} (\mathcal{M} \mathbf{1}_L)^{-1} \mathcal{M}, \quad \mathcal{M} = \mathrm{USAM} (X; d, W_\mathcal{Q}, W_\mathcal{K}).
\end{equation}
Here, $\mathbf{1}_L \in \mathbb{R}^L$ is a vector of ones and $\mathrm{diag} (\cdot)$ returns a diagonal matrix from a vector. 
The self-attention matrix is a row-normalized version of the unnormalized self-attention matrix.

Finally, \textit{self-attention} is defined as a parametrized mapping from $X$ into $\mathbb{R}^{L \times d}$ with parameters $W_\mathcal{Q}, W_\mathcal{K}, W_\mathcal{V} \in \mathbb{R}^{d_{hid} \times d}$. It has a form:
\begin{equation}
    \mathrm{SA} (X; d, W_\mathcal{Q}, W_\mathcal{K}, W_\mathcal{V}) = \mathrm{SAM} (X; d, W_\mathcal{Q}, W_\mathcal{K}) X W_\mathcal{V}.
\end{equation}
By the definition, all (normalized) self-attention matrices (\ref{eq:sam}) are \textit{right stochastic}, since their rows are nonnegative and sum to $1$.

Self-attention has the form of a differentiable dictionary, where the output at each position $1 \leq l \leq L$ is a sum of all values $W_\mathcal{V}^\top X_{l'}$, $1 \leq l' \leq L$, weighted proportionally to exponentiated dot products of the query $W_\mathcal{Q}^\top X_l$ and the key vectors $W_\mathcal{V}^\top X_{l'}$. Usually \cite{transformer}, these dot product are also divided by $\sqrt{d}$, since this empirically facilitates stable training. Without loss of generality, we do not include this linear scaling factor in our definition (\ref{eq:usam}), since it can be fused into one of the matrices $W_\mathcal{Q}$ or $W_\mathcal{K}$.

\section{Approximating sparse matrices by self-attention matrix} \label{sec:main}

\subsection{The main result} \label{sec:mainres}

We will call the square matrix \textit{$k$-nonzero-bounded}, if for each row or column of the matrix, the total number of nonzero elements is no more than $k$.

Apart from the notion of the bounded number of nonzero elements, we also  define matrices with elements of a bounded variation. For $\gamma \geq 1$, we call the matrix $A \in \mathbb{R}^{L \times L}$ with nonnegative elements \textit{$\gamma$-variation-bounded}, if for every row index $1 \leq i \leq L$ and every column indices $1 \leq j_1, j_2 \leq L$ such that $A_{i, j_1}, A_{i, j_2} \neq 0$,
\begin{equation}
    \gamma^{-1} \leq \frac{A_{i, j_1}}{A_{i, j_2}} \leq \gamma.
\end{equation}
For instance, all nonzero entries of a $1$-variation-bounded matrix are the same for each row of the matrix.

The following theorem is the main result of this paper:
\begin{theorem} \label{th:1}
Let $L > 1, k, d_{hid} \leq 2 L, d \leq d_{hid}$ be natural numbers, $d$ be even, $0 < \epsilon_1 < 1$, $0 < \epsilon_2 < \sqrt{2}$, $\gamma \geq 1$ be real numbers,
\begin{equation} \label{eq:dhidb}
    d \geq 32 \epsilon_2^{-2} k^2 \max (\log \gamma - \log \epsilon_1 + \epsilon_2, 1)^2 (2 \log L + \log (L - 1) + \log 2).    
\end{equation}
Then there exist $W_\mathcal{Q}, W_\mathcal{K} \in \mathbb{R}^{d_{hid} \times d}$, such that for any right stochastic, $k$-nonzero-bounded, $\gamma$-variation-bounded matrix $A \in \mathbb{R}^{L \times L}$, there is $X \in \mathbb{R}^{L \times d_{hid}}$ and $M = \mathrm{SAM} (X; d, W_\mathcal{Q}, W_\mathcal{K})$ satisfying
\begin{enumerate}
    \item For all row indices $1 \leq i \leq L$ and column indices $1 \leq j_1, j_2 \leq L$ such that $A_{i,j_1} = 0, A_{i, j_2} \neq 0$, it holds that
    \begin{equation} \label{eq:th1r1}
        \frac{M_{i,j_1}}{M_{i, j_2}} < \epsilon_1 ;
    \end{equation}
    \item For all row indices $1 \leq i \leq L$ and column indices $1 \leq j_1, j_2 \leq L$ such that $A_{i,j_1} \neq 0, A_{i, j_2} \neq 0$, it holds that
    \begin{equation} \label{eq:th1r2}
        \frac{A_{i,j_1}}{A_{i, j_2}} \cdot \exp(- \epsilon_2) < \frac{M_{i,j_1}}{M_{i, j_2}} < \frac{A_{i,j_1}}{A_{i, j_2}} \cdot \exp (\epsilon_2) .
    \end{equation}
\end{enumerate}
$W_\mathcal{Q}, W_\mathcal{K}$ can be constructed in $O(d_{hid} \cdot d)$ time. For any $A$, $X$ and $M$ can be computed in randomized time polynomial in $L, d_{hid}, k$.
\end{theorem}

Informally, Theorem \ref{th:1} states that for hidden sizes $d_{hid}, d$ \textbf{growing only logarithmically} with the sequence length $L$, there exist \textbf{fixed} parameter matrices $W_\mathcal{Q}, W_\mathcal{K}$ such that for any nonzero-bounded matrix $A$ there is a self-attention input $X$ such that $M = \mathrm{SAM} (X; d, W_\mathcal{Q}, W_\mathcal{K})$ approximates $A$ very well. The quality of approximation is characterized by upper and lower bounds on ratios of elements located in the same row of $M$:
\begin{enumerate}
    \item Equation (\ref{eq:th1r1}) means that zero elements of $A$ are approximated by elements of $M$ which are small compared to nonzero elements of the same row when $\epsilon_1$ is chosen small. By definition $M$ is a strictly positive matrix, therefore in principle we can only approximate zero elements of $A$ by very small positive numbers.
    \item Equation (\ref{eq:th1r2}) means that ratios of nonzero elements of the same row in $M$ are in a close multiplicative neighborhood of the corresponding ratios in $A$ when $\epsilon_2$ is chosen small. Since rows of both matrices $A$ and $M$ sum up to $1$, similar enough ratios of element pairs also imply element similarity in terms of their absolute magnitude.
\end{enumerate}

Finally, as the proof is constructive, we will obtain an algorithm for computing $W_\mathcal{Q}, W_\mathcal{K}$, which turn out to be matrices of a simple structure. For any $A$ from the theorem statement, the probabilistic algorithm induced by the proof enables $X$ and $M$ to be computed in randomized polynomial time in $L, d_{hid}, k$.

In the rest of the section we describe the detailed proof and intuition behind it.

\subsection{Proof of Theorem \ref{th:1}: matrix $B$ and the intuition behind the proof}

Define vector $A^{\min nz} \in \mathbb{R}^L$ so that for each row index $1 \leq i \leq L$, the minimal nonzero element in this row is $A^{\min nz}_i$. Define matrix $B \in \mathbb{R}^{L \times L}$ as follows. For all $1 \leq i, j \leq L$,
\begin{equation} \label{eq:bdef}
    B_{i,j} = \begin{cases}
    0 & \text{if } A_{i,j} = 0 ; \\
    \log A_{i,j} - \log A^{\min nz}_i - \log \epsilon_1 + \epsilon_2 & \text{otherwise}.
    \end{cases}
\end{equation}

Observe, that $C = \epsilon_1 \exp (-\epsilon_2) \mathrm{diag} (A^{\min nz}) \exp (B)$ can be thought of as an approximation of $A$:
\begin{equation} \label{eq:cdef}
    C = \epsilon_1 \exp (- \epsilon_2) \mathrm{diag} (A^{\min nz}) \exp (B) \approx A.
\end{equation}
Indeed, for any $1 \leq i, j \leq L$ such that $A_{i,j} \neq 0$, $C_{i,j} = A_{i,j}$ by definition of $B$ and $C$ (Equations \ref{eq:bdef}, \ref{eq:cdef}). On the other hand, when $A_{i,j} = 0$, $C_{i,j} = \epsilon_1 \exp(- \epsilon_2) A^{\min nz}_i \leq \epsilon_1$, where we use $\epsilon_2 > 0$ and $A^{\min nz}_i \leq 1$, since $A$ is a right stochastic matrix. Hence, smaller $\epsilon_1$ yields a better approximation (\ref{eq:cdef}).

Suppose we find $X, W_\mathcal{Q}, W_\mathcal{K}$ such that
\begin{equation} \label{eq:appr2}
    B \approx X W_\mathcal{Q} W_\mathcal{K}^\top X^\top .
\end{equation}

Intuitively, if the approximation (\ref{eq:appr2}) is sufficiently good 
then
\begin{gather}
    \mathrm{USAM} (X; d, W_\mathcal{Q}, W_\mathcal{K}) = \exp (X W_\mathcal{Q} W_\mathcal{K}^\top X^\top) \approx \exp (B) \nonumber \\
    = \epsilon_1^{-1} \exp(\epsilon_2) \mathrm{diag} (A^{\min nz})^{-1} C \approx \epsilon_1^{-1} \exp(\epsilon_2) \mathrm{diag} (A^{\min nz})^{-1} A , \label{eq:appr3}
\end{gather}
where in the last transition we use (\ref{eq:cdef}). Since the unnormalized self-attention matrix $\mathrm{USAM} (X; d, W_\mathcal{Q}, W_\mathcal{K})$ is a good approximation for $A$ with rescaled rows (recall $\epsilon_1^{-1} \exp(\epsilon_2) \mathrm{diag} (A^{\min nz})^{-1}$ multipliers), the normalized self-attention matrix $\mathrm{SAM} (X; d, W_\mathcal{Q}, W_\mathcal{K})$ should be a good approximation for $A$, which is itself row-normalized (right stochastic):
\begin{equation} \label{eq:appr4}
    \mathrm{SAM} (X; d, W_\mathcal{Q}, W_\mathcal{K}) \approx A.
\end{equation}

Next, we formally construct such $X, W_\mathcal{Q}, W_\mathcal{K}$ and derive tight error bounds for the approximation (\ref{eq:appr4}) in terms of matrix element ratios (\ref{eq:th1r1},\ref{eq:th1r2}).

\begin{figure}
    \centering
    \includegraphics[width=0.7\textwidth]{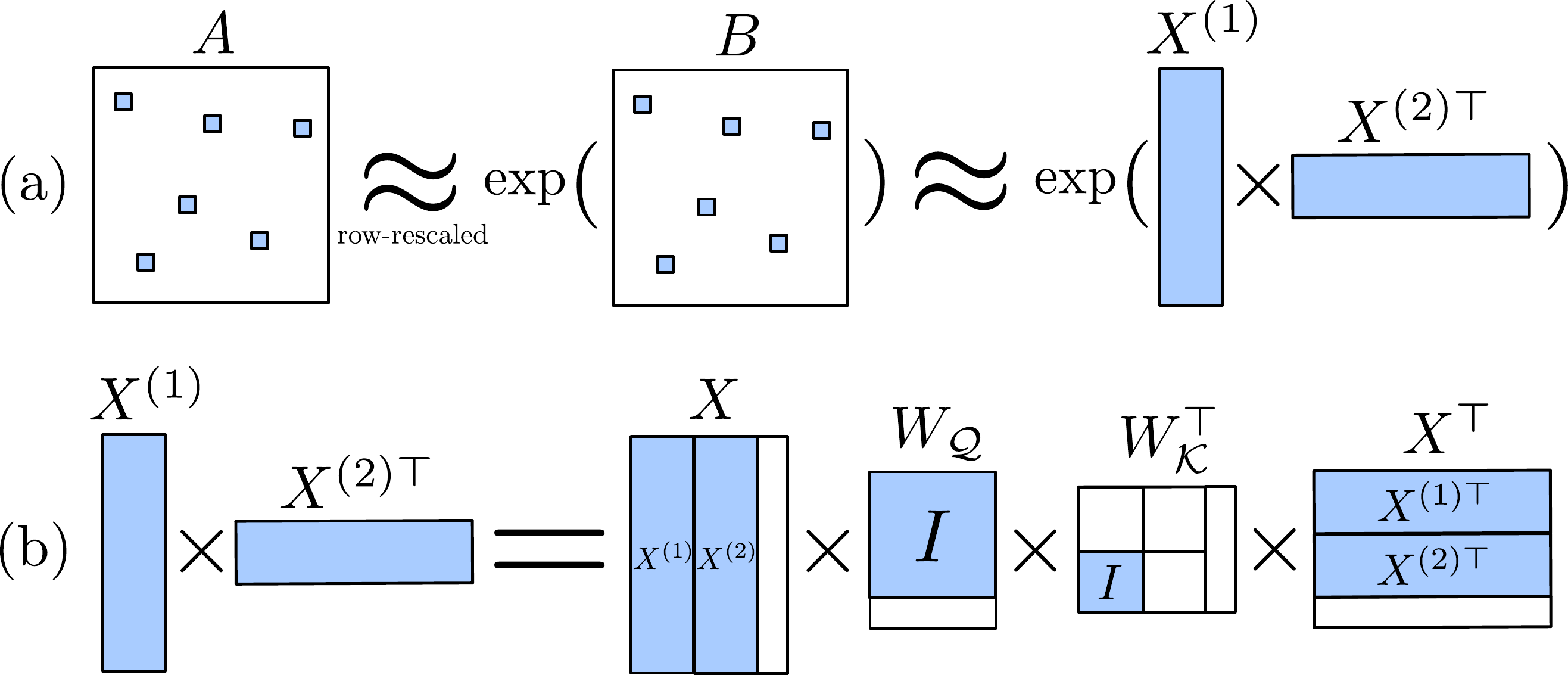}
    \caption{Illustration of the approximation scheme. \textbf{(a)} $\exp (B)$ is a row-rescaled approximation of $A$ (Equation \ref{eq:cdef}), whereas $X^{(1)} X^{(2) \top}$ is an unbiased approximation to $B$ (\ref{eq:unbiased}). \textbf{(b)} Representation of $X^{(1)} X^{(2) \top}$ as $X W_\mathcal{Q} W_\mathcal{K} X^\top$ according to (\ref{eq:xwdef}).}
    \label{fig:illustr}
\end{figure}




\subsection{Proof of Theorem \ref{th:1}: construction of $X, W_\mathcal{Q}, W_\mathcal{K}$ through random projections}

Consider a singular value decomposition (SVD) \cite{nla} of the matrix $B$: $B = U \Sigma V^\top$, where $U, V \in \mathbb{R}^{L \times L}$ are orthogonal matrices and $\Sigma = \mathrm{diag} (\sigma_1, \dots, \sigma_L)$, $\sigma_1 \geq \dots \sigma_L \geq 0$ are singular values of $B$. Define $D = U \Sigma$, then $B$ can be decomposed as $B = D V^\top$.

We will use random projections to compress $D$ and $V$ into matrices of shape $L \times d/2$. Namely, let $Y \in \mathbb{R}^{L \times d/2}$ be a random matrix sampled from a uniform distribution (Haar measure) on a set of Stiefel matrices\footnote{While $Y$ can be defined as a matrix with i.i.d. sub-Gaussian entries \cite{dotpr}, in general, orthogonal projections outperform unstructured ones in theory and practice \cite{ortho-features,performers,lin}. We also manage to obtain better dot product concentration results for Stiefel projections compared to unstructured ones (See discussion in Appendix \ref{app:lemma1proof}).} $\{ \Omega \in \mathbb{R}^{L \times d/2} | \Omega^\top \Omega = I_{d / 2} \}$. Here, $I_{d / 2}$ is a $(d/2) \times (d/2)$ identity matrix. Then we set $X^{(1)} = (2 L / d)^{1/2} D Y \in \mathbb{R}^{L \times d / 2}$, $X^{(2)} = (2 L / d)^{1/2} V Y \in \mathbb{R}^{L \times d / 2}$. $X^{(1)}, X^{(2)}$ can be considered compressions of $D, V$ since $X^{(1)} X^{(2) \top}$ is an unbiased approximation of $B = D V^\top$:
\begin{equation} \label{eq:unbiased}
    \mathbb{E} X^{(1)} X^{(2) \top} = D \times \mathbb{E} \left[ (2 L / d) \cdot Y Y^\top \right] \times V^\top = D \times \mathbb{E} \left[ L \cdot Y_{:,1} Y_{:, 1}^\top \right] \times V^\top = D V^\top = B,
\end{equation}
where we use the fact that columns of $Y$ are marginally uniformly distributed on $\mathcal{S}^{L - 1}$. See Figure \ref{fig:illustr}a for an illustration. We set $X, W_\mathcal{Q}, W_\mathcal{K}$ as
\begin{gather} \label{eq:xwdef}
    X = \begin{bmatrix} X^{(1)} \! & \! X^{(2)} \! & \! \mathbf{0}_{L \times (d_{hid} - d)} \end{bmatrix}, \, W_\mathcal{Q} \! = \! \begin{bmatrix} I_d \! & \!\! \mathbf{0}_{d \times (d_{hid} - d)} \end{bmatrix}^\top, \, W_\mathcal{K} \! = \! \begin{bmatrix} \Omega_d \! & \! \mathbf{0}_{d \times (d_{hid} - d)} \end{bmatrix}^\top, \\
    \Omega_d = \begin{bmatrix} \mathbf{0}_{d/2 \times d/2} & I_{d / 2} \\ \mathbf{0}_{d/2 \times d/2} & \mathbf{0}_{d/2 \times d/2} \end{bmatrix}, \nonumber
\end{gather}
where $\mathbf{0}_{\dots \times \dots}$ denotes a zero matrix of the corresponding shape. It is easy to see that in this case $X W_\mathcal{Q} W_\mathcal{K}^\top X^\top = X^{(1)} X^{(2) \top}$ (see Figure \ref{fig:illustr}b).

Our next step is to prove that with a nonzero probability, differences of elements in $X W_\mathcal{Q} W_\mathcal{K}^\top X^\top$ concentrate near the same differences in $B$:
\begin{lemma} \label{lemma:1}
With probability greater than $(L + 2)^{-1}$ it holds that
\begin{equation} \label{eq:evb}
    \forall 1 \leq i, j_1, j_2 \leq L, j_1 \neq j_2: | (X W_\mathcal{Q} W_\mathcal{K}^\top X^\top)_{i,j_1} - (X W_\mathcal{Q} W_\mathcal{K}^\top X^\top)_{i,j_2} - B_{i,j_1} + B_{i,j_2} | < \epsilon_2.
\end{equation}
\end{lemma}
The proof (Appendix \ref{app:lemma1proof}) uses a corollary of the seminal  Johnson-Lindenstrauss lemma \cite{jlt} about inner product preservation under random projections \cite{dotpr}. Two crucial observations are that
\begin{itemize}
    \item $W_\mathcal{Q}$ and $W_\mathcal{K}$ do not depend on $A$ by construction (\ref{eq:xwdef});
    \item according to Lemma \ref{lemma:1}, $X$ satisfying (\ref{eq:evb}) can be found with any probability by redrawing $Y$ $O(L)$ times.
\end{itemize}

Suppose that (\ref{eq:xwdef}) holds. Then for any $1 \leq i, j_1, j_2 \leq L$, $j_1 \neq j_2$:
\begin{equation} \label{eq:xwbd}
    B_{i,j_1} - B_{i,j_2} - \epsilon_2 < (X W_\mathcal{Q} W_\mathcal{K}^\top X^\top)_{i,j_1} - (X W_\mathcal{Q} W_\mathcal{K}^\top X^\top)_{i,j_2} < B_{i,j_1} - B_{i,j_2} + \epsilon_2.
\end{equation}
By definition of $B$ (\ref{eq:bdef}), whenever $A_{i,j_1} = 0, A_{i, j_2} \neq 0$, the right hand side inequality in (\ref{eq:xwbd}) is rewritten as
\begin{gather}
    (X W_\mathcal{Q} W_\mathcal{K}^\top X^\top)_{i,j_1} - (X W_\mathcal{Q} W_\mathcal{K}^\top X^\top)_{i,j_2} < - \log A_{i,j_2} + \log A^{\min nz}_i + \log \epsilon_1 \leq \log \epsilon_1. \label{eq:zbnd}
\end{gather}
Here we also used $A_{i,j_2} \geq A^{\min nz}_i$. (\ref{eq:zbnd}) is equivalent to (\ref{eq:th1r1}) after exponentiating, since exponents of $(X W_\mathcal{Q} W_\mathcal{K}^\top X^\top)_{i,j_1}$ and $(X W_\mathcal{Q} W_\mathcal{K}^\top X^\top)_{i,j_2}$ are $M_{i,j_1}$ and $M_{i,j_2}$ rescaled by the same factor.

Similarly to (\ref{eq:zbnd}), whenever $A_{i,j_1}, A_{i, j_2} \neq 0$, by expanding $B$'s definition, (\ref{eq:xwbd}) is rewritten as
\begin{equation} \label{eq:diffb}
    \! \log A_{i,j_1} \! - \log A_{i,j_2} \! - \epsilon_2 \! < \! (X W_\mathcal{Q} W_\mathcal{K}^\top X^\top)_{i,j_1} \! - (X W_\mathcal{Q} W_\mathcal{K}^\top X^\top)_{i,j_2} \! < \log A_{i,j_1} \! - \! \log A_{i,j_2} + \epsilon_2,
\end{equation}
which is equivalent to (\ref{eq:th1r2}) after exponentiating. The proof of Theorem \ref{th:1} is concluded. \qedsymbol

\subsection{The case of causal self-attention}

Another very popular type of self-attention mechanism is \textit{causal self-attention}, when each position $i$ only attends to elements $j \leq i$. This modification is required for autoregressive language modelling \cite{gpt-2,gpt-3} when each token is modelled as depending only on previous tokens in the sequence. We define the causal self-attenion matrix $\mathrm{CSAM}$ and causal self-attention $\mathrm{CSA}$ 
as
\begin{gather*}
    \mathrm{CSAM} (X; d, W_\mathcal{Q}, W_\mathcal{K}) = \mathrm{diag} (\mathcal{M}' \mathbf{1}_L)^{-1} \mathcal{M}', \quad \mathcal{M}' = \mathrm{tril}(\mathrm{USAM} (X; d, W_\mathcal{Q}, W_\mathcal{K})), \\
    \mathrm{CSA} (X; d, W_\mathcal{Q}, W_\mathcal{K}, W_\mathcal{V}) = \mathrm{CSAM} (X; d, W_\mathcal{Q}, W_\mathcal{K}) X W_\mathcal{V},    
\end{gather*}
where $\mathrm{tril} (\cdot)$ is the \textit{lower triangular part} of the argument matrix, meaning that it zeroes out all elements strictly above the main diagonal.

A natural question is whether the analog of Theorem \ref{th:1} holds for causal self-attention matrices. Since these matrices are lower triangular, we should only attempt to approximate lower-triangular right-stochastic matrices $A$. In fact, we obtain the following result. 
\begin{corollary} \label{cor:1}
    Change Theorem \ref{th:1} as follows: 1) require $A$ to be a lower triangular matrix (along with other requirements on $A$), 2) replace $\mathrm{SAM} \to \mathrm{CSAM}$ and 3) consider indices $j_1, j_2$ to be upper bounded by $i$ instead of $L$: $1 \leq j_1, j_2 \leq i$. The obtained new statement is true.
\end{corollary}
\begin{proof}
The proof is unchanged compared to the proof of Theorem \ref{th:1} with the only change that $j_1, j_2$ are considered in the range $1 \leq j_1, j_2 \leq i$ when computing difference bounds (\ref{eq:zbnd},\ref{eq:diffb}). Cases when column indices $j_1$ or $j_2$ are bigger than the row index $i$ are redundant, since both $A$ and $M = \mathrm{CSAM} (X; d, W_\mathcal{Q}, W_\mathcal{K})$ have zero entries above the main diagonal.
\end{proof}
We conclude that the same logarithmic dependence $d = O(\log L)$ holds for the causal self-attention.

\begin{figure}
    \centering
    \includegraphics[width=0.95\textwidth]{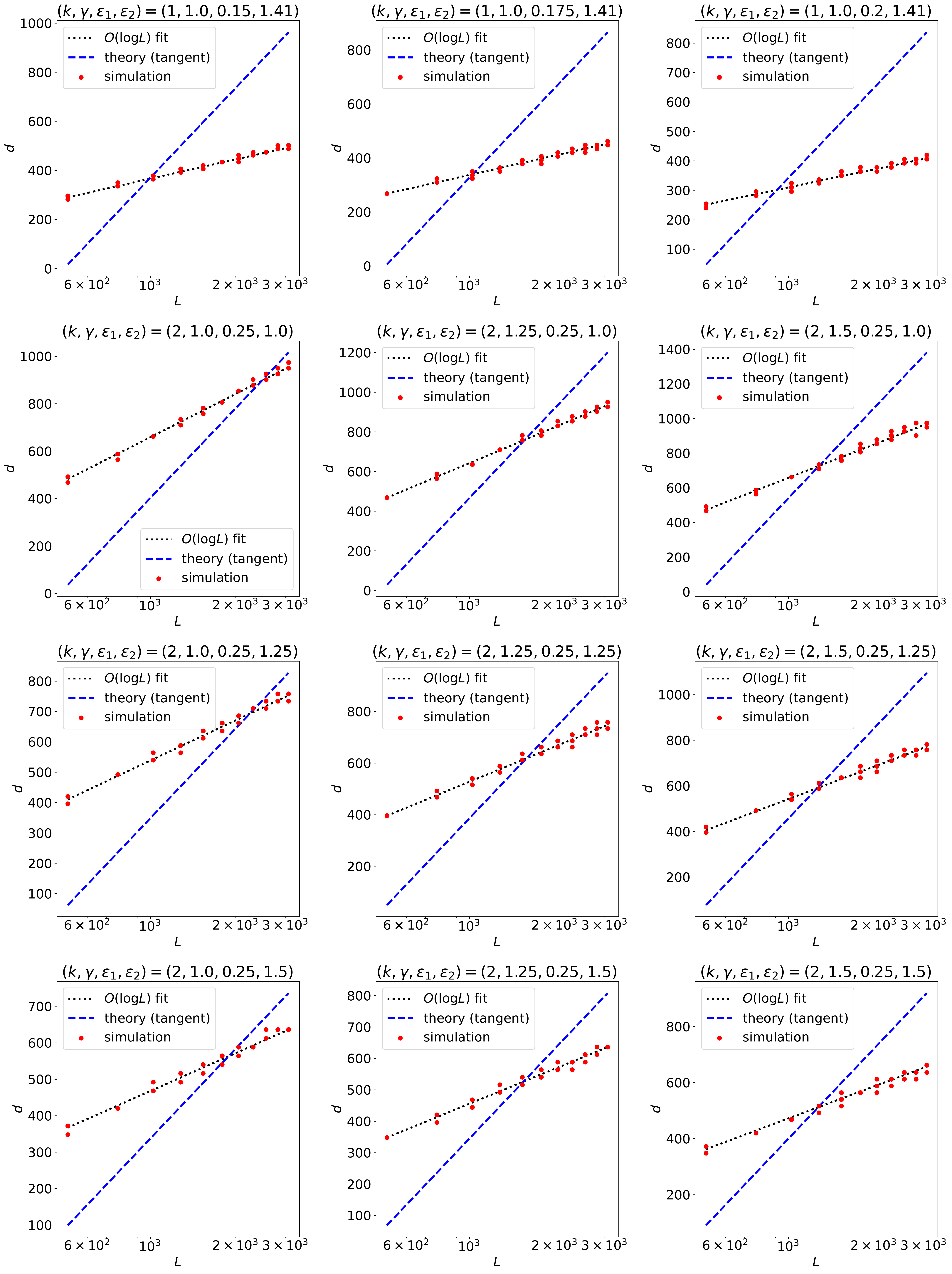}
    \caption{Finding empirical dependency of $d_{\min} (\epsilon_1, \epsilon_2)$ on $L$ given the fixed set of parameters $(k, \gamma, \epsilon_1, \epsilon_2)$. Each plot corresponds to one out of twelve sets of tested parameters. Red circles correspond to simulation results: we redraw  matrix $A$ 5 times for each $L$, resulting in 5 red circles per $L$. The horizontal $L$ axis is in a logarithmic scale, so that $O(\log L)$ corresponds to a straight line. The black dotted line corresponds to a $O (\log L)$ fit for the dots (linear when x axis scale is logarithmic). The blue dashed line indicates the tangent of the theoretical upper bound on $d_{\min} (\epsilon_1, \epsilon_2)$ (right hand side in Equation \ref{eq:dhidb}). We experiment with $k$ from $\{ 1, 2 \}$, where for $k = 1$ we set $\gamma = 1$, $\epsilon_2 = 1.41$, since for each row of $A$ there is a single nonzero value and (\ref{eq:th1r2}) is true for any $0 < \epsilon_2 < \sqrt{2}$.}
    \label{fig:plots}
\end{figure}

\section{Experiments} \label{sec:exp}

Theorem \ref{th:1} 
suggests an upper bound (r.h.s. in Equation \ref{eq:dhidb}) for the $d_{\min} (\epsilon_1, \epsilon_2)$ -- i.e. the minimal $d$ which results in $M$ satisfying (\ref{eq:th1r1},\ref{eq:th1r2}) for fixed $\epsilon_1, \epsilon_2$. A question which we address in the experimental section is, therefore, "\textbf{What is the actual $\boldsymbol{d_{\min} (\epsilon_1, \epsilon_2)}$ in practice? Does it satisfy the logarithmic law $\boldsymbol{d = O(\log L)}$}?"

To answer this question, we perform the following simulation. We select a range of $(k, \gamma, \epsilon_1, \epsilon_2)$ parameters. For each set of parameters, we iterate over $L$ on a uniform grid from 512 to 3072 with a step size 256. For each $L$ we sample the matrix $A$ and iterate over a uniform grid of $d$ values in ascending order until we find such $d$ which results in $M$ satisfying (\ref{eq:th1r1},\ref{eq:th1r2}). We sample $A$ by doing two passes over elements of the matrix. The first pass is over randomly permuted row indices and then randomly permuted column indices in a nested loop. The second one is over randomly permuted column indices and randomly permuted row indices in a nested loop. At each position, we add a new nonzero value if that does not violate $k$-nonzero-bounded condition. The nonzero value is set to either $1$ or $\gamma$ by a coin flip.

To check whether for the current $d$ there is $M$ satisfying (\ref{eq:th1r1},\ref{eq:th1r2}), we construct $Y$, $X^{(1)}$, $X^{(2)}$ and $M$ using the algorithm implied by the proof of Theorem \ref{th:1}. To sample Stiefel matrices $Y$, we use the algorithm based on QR decompositions of random Gaussian matrices from \cite{rorth}. We redraw the $Y$ matrix $Q L$ times, $Q = 1$, in the spirit of Lemma \ref{lemma:1} suggesting that $O(L)$ redraws should be enough to find the right $Y, X^{(1)}, X^{(2)}, X$ with a constant probability (when $d$ is big enough).

Figure \ref{fig:plots} illustrates the results. A remarkable observation is that, although empirical $d_{\min} (\epsilon_1, \epsilon_2)$ (red circles) grows slower than the theoretical upper bound (blue dashed line), it nicely fits the logarithmic curve $d = O(\log L)$ (black dotted line) in all twelve evaluated setups. The fact that the true $d_{\min} (\epsilon_1, \epsilon_2)$ grows slower than (\ref{eq:dhidb}) is natural, since (\ref{eq:dhidb}) is an upper bound on it. Though, as experiments reveal, both curves differ only by a constant multiplicative factor.

We run an additional experiment to reveal how $d_{\min} (\epsilon_1, \epsilon_2)$ depends on the number of samples $Q L$ used to find $M$ satisfying (\ref{eq:th1r1},\ref{eq:th1r2}). We take 2 out of 12 setups from Figure \ref{fig:plots} and try a range of $Q$ values from $0.1$ to $5.0$. Results are illustrated on Figure \ref{fig:sf_plots}. We observe that $d_{\min} (\epsilon_1, \epsilon_2)$ does not depend a lot on the choice of $Q$ and is roughly unchanged. Therefore, we conclude that our findings regarding the behaviour of empirical $d_{\min} (\epsilon_1, \epsilon_2)$ do not depend on $Q$ much and $Q = 1$ is a reasonable choice.

Finally, we visually assess the attention map produced by our Algorithm (Figure \ref{fig:amap}). We make sure that the generated self-attention matrix $M$ has the same (approximate) sparsity pattern as $A$. Additional experimental details and results (grids used to find $d$ in Figure \ref{fig:plots}, more attention maps like Figure \ref{fig:amap}) can be found in Appendix \ref{app:expdet}.

\begin{figure}[ht!]
    \centering
    \includegraphics[width=0.9\textwidth]{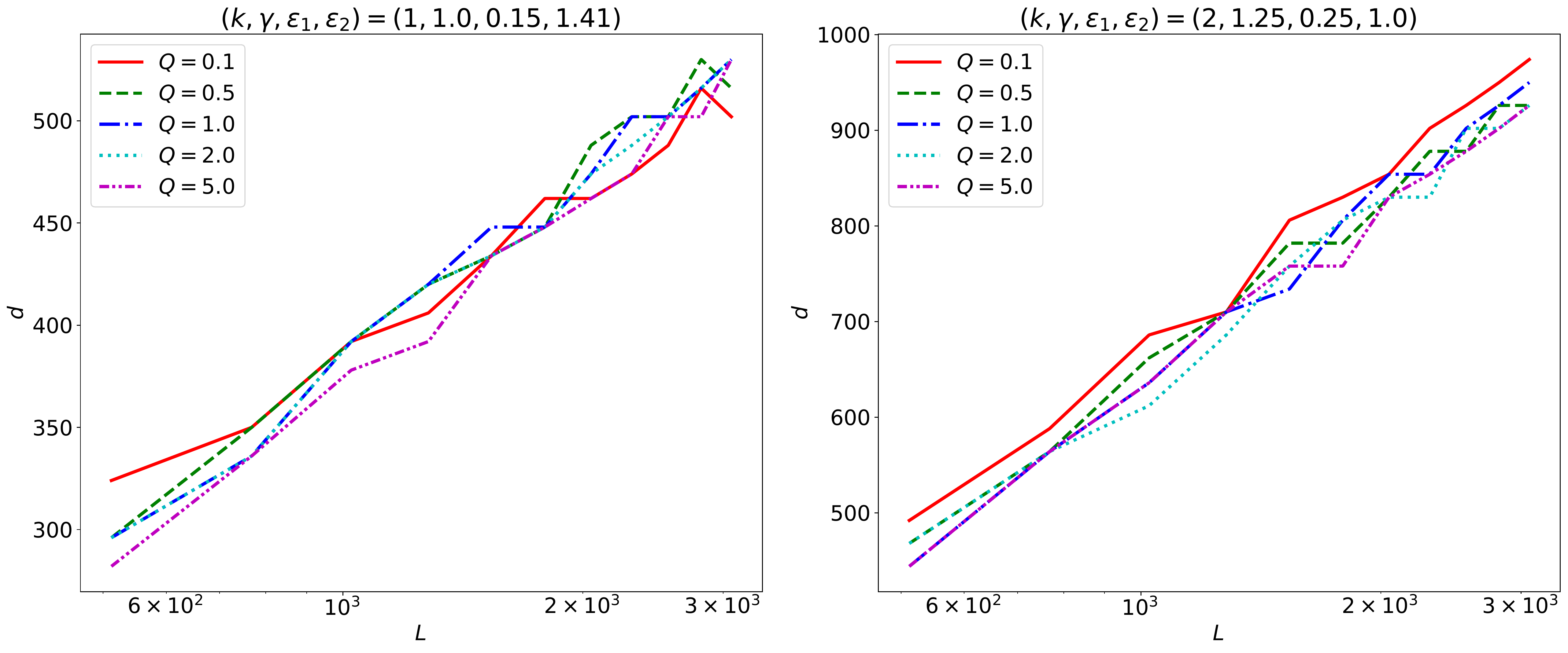}
    \caption{Empirical dependence of $d_{\min} (\epsilon_1, \epsilon_2)$ on $Q$ -- the factor defining the number of samples $QL$ (rounded to an integer) used to find the right matrix $M$. We use 2 out of 12 parameter sets from Figure \ref{fig:plots} (see plot titles). For each $Q$ we repeat the procedure to generate empirical values of $d_{\min} (\epsilon_1, \epsilon_2)$ (red circles from Figure \ref{fig:plots}) and connect them into a line for the better visualization.}
    \label{fig:sf_plots}
\end{figure}

\begin{figure}[ht!]
    \centering
    \includegraphics[width=0.8\textwidth]{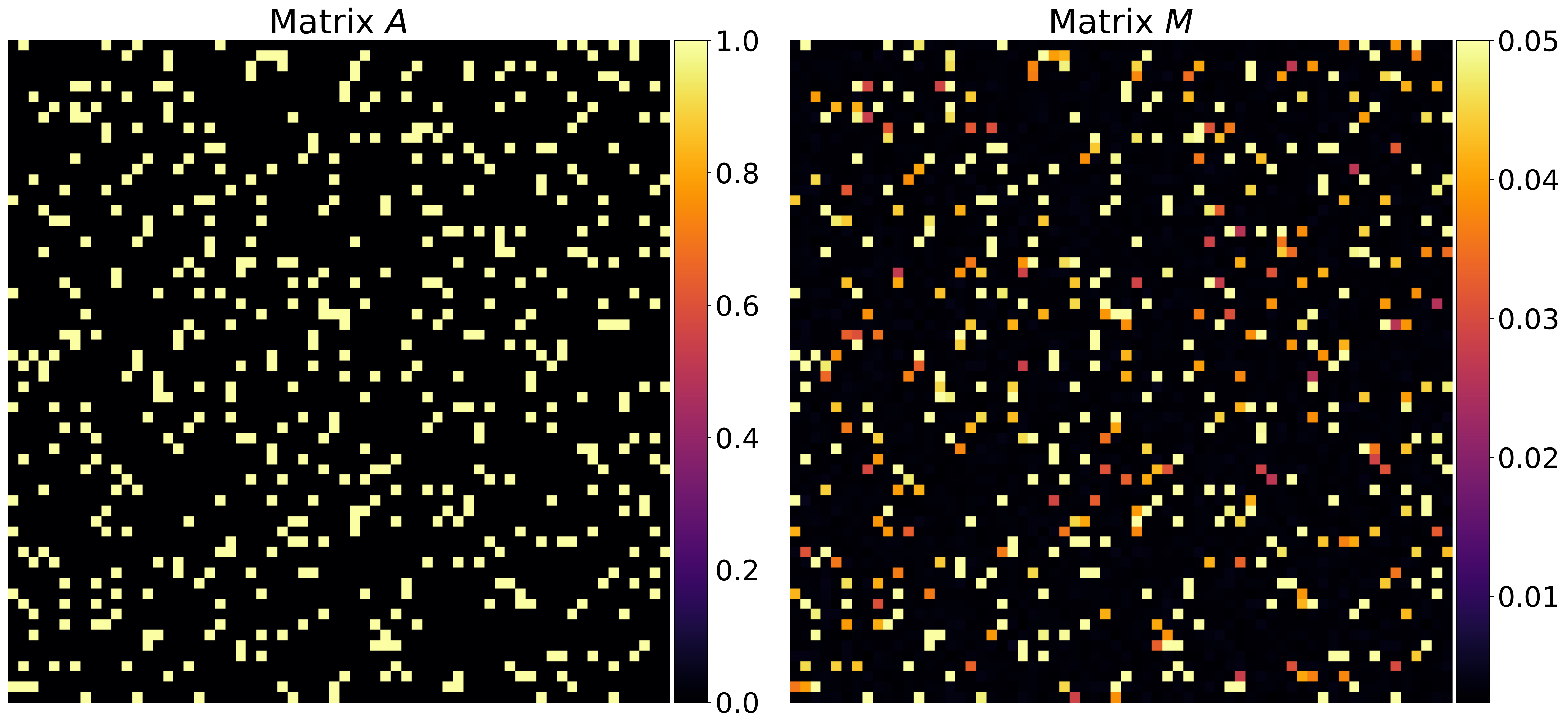}
    \caption{Attention maps $M$ generated by our algorithm from the proof of Theorem \ref{th:1}. We use $(k, \gamma, \epsilon_1, \epsilon_2) = (1, 1.0, 0.15, 1.41)$, $L = 512$, $d = 300$ (i.e. all nonzero elements are exactly $1$) and redraw $M$ until it satisfies (\ref{eq:th1r1},\ref{eq:th1r2}). We generate $A$ by the same algorithm used for Figure \ref{fig:plots}.  For the better visualization, we downscale $512 \times 512$ matrices into $64 \times 64$ images by taking the maximum over each square of size $8 \times 8$. Also, we clip entries of $M$ by $0.05$. Sparsity patterns visually coincide for both $A$ and $M$. More images can be found in Appendix \ref{app:expdet}.}
    \label{fig:amap}
\end{figure}





\section{Related work}

\textbf{Expressive power of Transformers.} As Transformers gained popularity, more theoretical results have emerged to explain their expressive power. Transformers were shown to be universal approximators \cite{truapp}, Turing-complete \cite{turingtr} and able to recognize counter languages \cite{cntlng}. Furthermore, Transformer modifications such as BigBird \cite{bigbird}, Transformers with hard attention \cite{hardatt} and sparse Transformers \cite{sptr} were shown to be universal approximators. 
Note that \cite{truapp,turingtr} rely on multilayer constructions, whereas we consider a single self-attention module, and \cite{bigbird,hardatt,sptr} analyze nonconventional forms of self-attention. \cite{pureatt} analyze limitations of a pure self-attention Transformer, i.e. without feedforward blocks and skip connections. In \cite{convsa}, authors show that the multi-head self-attention can provably model any image convolution layer. Again, this is hard to compare directly to our work since the construction in \cite{convsa} relies on a multiple number of self-attention heads proportional to the input size. Perhaps the most relevant to our work is \cite{lrbottleneck} where the authors show that for \textit{large} $d$ ($d \geq L$) and fixed inputs there exist $W_\mathcal{Q}, W_\mathcal{K}$ which approximate any positive right stochastic matrix via self-attention. 
In contrast, we analyze expressive power when $d$ is very small ($d = O(\log L)$).

\textbf{Random projections and Johnson-Lindenstrauss lemma.} Our proof techniques rely on the seminal Johnson-Lindenstrauss tranformation (JLT) \cite{jlt} used for dimensionality reduction \cite{dimr}. A random projection approach similar to ours was used in \cite{spher} to lower-bound graph sphericity -- a characteristic which is NP-hard to compute in general. A related \textit{random features} technique, relying on random projections, was originally introduced to improve efficiency of kernel SVMs \cite{rfs}, but recently found application in speeding up long-sequence Transformers \cite{performers}. We use Stiefel matrices as random projections, which in general result in tighter approximations than unconstrained projections \cite{ortho-features,performers,lin}. Ensembles of orthogonal random projections were shown to provide much better concentration results for the estimators relying on them in various other contexts, in particular: kernel approximation \cite{recycling_randomness, ortho-features, spinners, binary_embeddings} (JLT can be considered a special instantiation with a dot-product kernel), estimation of the gradients of Gaussian smoothings with evolution strategy methods \cite{structured_es},
kernel ridge regression techniques \cite{initialization_matters}, sliced Wasserstein distance estimation \cite{wasserstein} and more.
\section{Limitations and negative societal impacts} \label{sec:lim}
One limitation of this work, which also holds for many other theoretical results on deep learning, 
is that strict assumptions need to be made in order to proceed with theoretical reasoning. For instance, the assumptions of $k$-nonzero-bounded and $\gamma$-variation-bounded approximated matrix $A$ may be restrictive in certain cases. We hope to explore and alleviate these assumptions in future work.

This work theoretically studies Transformer networks -- models used in various applications. These models can have the following negative societal impacts: large $\text{CO}_2$ emissions during training \cite{co2}, privacy and data leak vulnerabilities \cite{gptleak}, bias and fairness issues and malicious misuse \cite{gpt-3,stochpar}.

\section{Conclusion}
We have analyzed the expressiveness of the self-attention matrix as a mechanism to approximate sparse patterns.
In practice, self-attention patterns are typically sparse and dynamic, meaning that they depend on the input to the network, while weights of the self-attention module are fixed. We aim to quantify expressiveness of self-attention.
Using random projection techniques and the seminal Johnson-Lindenstrauss lemma, we prove that weights of self-attention can be constructed in such a way that any sparse matrix can be approximated with certain input to the self-attention module. Sparsity is characterized by a bounded small number of nonzero values at each row or column of the approximated matrix.
We show that, when error and other parameters are fixed, $d$ grows only logarithmically with the sequence length $L$ ($d = O( \log L )$). We hope our work will facilitate further in-depth theoretical analysis of self-attention and Transformers to 
understand better their remarkable performance across a variety of tasks.

\section{Acknowledgments}

Valerii Likhosherstov acknowledges support from the
Cambridge Trust. Adrian Weller acknowledges support from a Turing AI Fellowship under grant EP/V025379/1, The Alan Turing Institute under EPSRC grant EP/N510129/1 and TU/B/000074, and the Leverhulme Trust via CFI.


\bibliographystyle{plain}
\bibliography{references}


\newpage
\appendix

\section{New concentration bounds for orthogonal Johnson-Lindenstrauss transform and proof of Lemma \ref{lemma:1}}

First, we derive new concentration results for the Johnson-Lindenstrauss transform (JLT) with orthogonal (Stiefel) projections, which are used in the proof of Lemma \ref{lemma:1}. After that, we prove the lemma.

\subsection{Theory of the orthogonal JLT with deterministic-lengths projections}

Note that in the JLT-part of our algorithm instead of applying i.i.d. Gausian projections sampled from $\mathcal{N}(0,\frac{1}{\sqrt{m}} I_p)$, where $p = L$ and $m = d / 2$ is the number of projections (as in the standard JLT setting), we instead choose different projections to be exactly orthogonal and a fixed projection to be chosen from the renormalized Gaussian distribution
$\sqrt{\frac{p}{m}}\frac{g}{\|g\|_{2}}$, where $g \sim \mathcal{N}(0,I_{p})$. We call this version of the JLT, \textit{Orthogonal JLT with Deterministic-Length Projections}, or simply: OJLTD. We now show that these modifications:
\begin{itemize}
\item projections renormalization to \textbf{deterministic-length} vectors and
\item orthogonalization of the ensemble of projections,
\end{itemize}
that consitute OJLTD, lead to \textbf{strictly better} concentration results of the resulting estimator than standard JLT. To see this intuitively, notice that if $m=p$ then OJLTD maps a vector to its representation in the randomly rotated coordinate system, in particular it exactly preserves distances and dot-products, which cannot be said about standard JLT even if $m=p$ (of course in practice we are interested in the setting where $m \ll p$).
For $m \leq p$, denote by $g_{1}^{\mathrm{ort}}$,...,$g_{m}^{\mathrm{ort}}$ the orthogonal ensemble of vectors such that $g_{i}^{\mathrm{ort}} \sim \sqrt{p} \frac{g}{\|g\|_{2}}$ and $g \sim \mathcal{N}(0,I_{p})$ (thus we have: $g_{i}^{\mathrm{ort}} \perp g_{j}^{\mathrm{ort}}$ for $i \neq j$).

Our result is as follows:
\begin{theorem}[Dot Product under Orthogonal Random Projections]
\label{ortho-theorem}
Let $x,y \in \mathbb{R}^{p}$. Let $R^{\mathrm{ort}} \in \mathbb{R}^{m \times p}$ be a random projection matrix with rows: 
$(\omega_{1}^{\mathrm{ort}})^{\top}=\sigma g_{1}^{\mathrm{ort}},...,(\omega_{m}^{\mathrm{ort}})^{\top}=\sigma g_{m}^{\mathrm{ort}} \in \mathbb{R}^{p}$ for some $\sigma >0$. 
Take $0 < \epsilon < 1$.
Then the following holds:
\begin{equation}
\label{lemma-ineq-1and2}
\mathbb{P}[|(R^{\mathrm{ort}}x)^{\top}R^{\mathrm{ort}}y - m\sigma^{2}x^{\top}y| < \epsilon m\sigma^{2}\|x\|_{2}\|y\|_{2}] < \left(2-\frac{2}{p+2}\right)\exp\left(-\frac{m\epsilon^{2}}{8}\right),
\end{equation}
\end{theorem}

To the best of our knowledge, this is the first result showing that orthogonal Gaussian ensembles improve exponentially small bounds for standard JLT using independent projections (previous results for the orthogonal Gaussian projections focused on improving mean squared error bounds \cite{ortho-features}). Interestingly, we show that this result can be straightforwardly derived by modifying the proof of Theorem 2 from \cite{performers} about softmax kernel estimation via random projections
and by leveraging the fact that the newly introduced in \cite{performers} \textit{regularized softmax kernel} is upper bounded by the regular softmax kernel. Both kernels can be thought of as moment generating functions corresponding to distributions involving deterministic-length and random-length projections respectively and the relationship between them induced relationships between $\Sigma$-parameters for the corresponding sub-Gaussian distributions. 

Even though the concentration presented in Theorem \ref{ortho-theorem} is certainly not tight (since it does not converge to perfect estimation for $m=p$), it is \textbf{strictly better} that the one for the standard JLT mechanism applying matrices $R$ with i.i.d rows of the form $\sigma g$ for $g \sim \sigma \mathcal{N}(0,I_{D})$, which is of the form (see: Theorem 2.1 from \cite{dotpr}): 
\begin{equation}
\label{lemma-ineq-1and2}
\mathbb{P}[|(Rx)^{\top}Ry - m\sigma^{2}x^{\top}y| < \epsilon m\sigma^{2}\|x\|_{2}\|y\|_{2}] < 2\exp(-\frac{m\epsilon^{2}}{8}).
\end{equation}
\begin{proof}
Denote: $\rho(\epsilon) = (1+\frac{1-\epsilon}{1+\epsilon}\frac{\|x+y\|_{2}^{2}}{\|x-y\|_{2}^{2}})^{-1}$.
We will prove the following two inequalities that, while combined, lead to our main result:
\begin{equation}
\label{lemma-ineq-1}
\mathbb{P}[(R^{\mathrm{ort}}x)^{\top}R^{\mathrm{ort}}y < m\sigma^{2}x^{\top}y-\epsilon m\sigma^{2}\|x\|_{2}\|y\|_{2}] < (1-\frac{2\rho(\epsilon)}{d+2})\exp(-\frac{m\epsilon^{2}}{8}),
\end{equation}
\begin{equation}
\label{lemma-ineq-2}
\mathbb{P}[(R^{\mathrm{ort}}x)^{\top}R^{\mathrm{ort}}y > m\sigma^{2}x^{\top}y+\epsilon m\sigma^{2}\|x\|_{2}\|y\|_{2}] < (1-\frac{2(1-\rho(\epsilon))}{d+2})\exp(-\frac{m\epsilon^{2}}{8}).
\end{equation}

Our proof, as mentioned above, heavily relies on the proof of Theorem 2 from \cite{performers}, yet we give all the details below for Reader's convenience and since it is not exactly the same (in particular applying also other techniques).
Define: $S_{i}^{\mathrm{ort}}=((\omega^{\mathrm{ort}}_{i})^{\top}z)^{2}$, where $z=x-y$. Similarly, define: $S_{i}^{\mathrm{iid}}=((\omega^{\mathrm{iid}}_{i})^{\top}z)^{2}$, where $\omega^{\mathrm{iid}}_{1},...,\omega^{\mathrm{iid}}_{m} \overset{\mathrm{iid}}{\sim} \mathcal{N}(0,\sigma I_{p})$ and thus
$(\omega^{\mathrm{iid}}_{1})^{\top},...,(\omega^{\mathrm{iid}}_{m})^{\top}$ are the rows of the regular JLT projection matrix $R=R^{\mathrm{iid}}$.
Now take some $\theta>0$.

We have: 
\begin{align}
\begin{split}
\mathbb{E}[e^{\theta (S_{1}^{\mathrm{ort}}+...+S_{m}^{\mathrm{ort}})}] = \mathbb{E}[\sum_{j=0}^{\infty} \frac{(\theta \sum_{i=1}^{m} S_{i}^{\mathrm{ort}})^{j}}{j!}] 
= \mathbb{E}[\sum_{j=0}^{\infty}\frac{\theta^{j}}{j!}(\sum_{i=1}^{m} S^{\mathrm{ort}}_{i})^{j}]=\\
\sum_{j=0}^{\infty}\frac{\theta^{j}}{j!} \mathbb{E}[(\sum_{i=1}^{m} S^{\mathrm{ort}}_{i})^{j}]=
\sum_{j=0}^{\infty}\frac{\theta^{j}}{j!}
\mathbb{E}[\sum_{(j_{1},...,j_{m}) \in \mathcal{S}_{j}} c(j_{1},...,j_{m}) (S_{1}^{\mathrm{ort}})^{j_{1}} \cdot ... \cdot (S_{m}^{\mathrm{ort}})^{j_{m}}],
\end{split}
\end{align}
where $\mathcal{S}_{j} = \{(j_{1},...,j_{m}) \in \mathbb{N} \times ...\times \mathbb{N}:j_{1},...,j_{m} \geq 0, j_{1}+...+j_{m}=j\}$ and 
for some positive constants $c(j_{1},...,j_{m})$.

Thus we have:
\begin{equation}
\mathbb{E}[e^{\theta (S_{1}^{\mathrm{ort}}+...+S_{m}^{\mathrm{ort}})}] = \sum_{j=0}^{\infty} \frac{\theta^{j}}{j!} \sum_{(j_{1},...,j_{m}) \in \mathcal{S}_{j}} c(j_{1},...,j_{m}) \mathbb{E}[(S_{1}^{\mathrm{ort}})^{j_{1}} \cdot ... \cdot (S_{m}^{\mathrm{ort}})^{j_{m}}].  
\end{equation}

Similarly, we get:
\begin{equation}
\mathbb{E}[e^{\theta (S_{1}^{\mathrm{iid}}+...+S_{m}^{\mathrm{iid}})}] = \sum_{j=0}^{\infty} \frac{\theta^{j}}{j!} \sum_{(j_{1},...,j_{m}) \in \mathcal{S}_{j}} c(j_{1},...,j_{m}) \mathbb{E}[(S_{1}^{\mathrm{iid}})^{j_{1}} \cdot ... \cdot (S_{m}^{\mathrm{iid}})^{j_{m}}].    
\end{equation}

Therefore we get:
\begin{align}
\begin{split}
\Delta = \mathbb{E}[e^{\theta (S_{1}^{\mathrm{iid}}+...+S_{m}^{\mathrm{iid}})}] - \mathbb{E}[e^{\theta (S_{1}^{\mathrm{ort}}+...+S_{m}^{\mathrm{ort}})}] \\
=
\sum_{j=0}^{\infty} \frac{\theta^{j}}{j!} \sum_{(j_{1},...,j_{m}) \in \mathcal{S}_{j}} c(j_{1},...,j_{m}) \left(\mathbb{E}[(S_{1}^{\mathrm{iid}})^{j_{1}} \cdot ... \cdot (S_{m}^{\mathrm{iid}})^{j_{m}}] - \mathbb{E}[(S_{1}^{\mathrm{ort}})^{j_{1}} \cdot ... \cdot (S_{m}^{\mathrm{ort}})^{j_{m}}]\right)
\end{split}
\end{align}
Thus we obtain:
\begin{align}
\begin{split}
\Delta = \sum_{j=0}^{\infty} \frac{\theta^{j}}{j!} \sum_{(j_{1},...,j_{m}) \in \mathcal{S}_{j}} c(j_{1},...,j_{m})  \widehat{\Delta}(j_{1},...,j_{m}),
\end{split}    
\end{align}
and $\widehat{\Delta}(j_{1},...,j_{m})$ is given as:
\begin{align}
\begin{split}
\label{imp-ineq}
\widehat{\Delta}(j_{1},...,j_{m}) = \mathbb{E}[((\omega_{1}^{\mathrm{iid}})^{\top}z)^{2j_{1}} \cdot ... \cdot ((\omega_{m}^{\mathrm{iid}})^{\top}z)^{2j_{m}}] - 
\mathbb{E}[((\omega_{1}^{\mathrm{ort}})^{\top}z)^{2j_{1}} \cdot ... \cdot ((\omega_{m}^{\mathrm{ort}})^{\top}z)^{2j_{m}}].
\end{split}
\end{align}
Our next goal is to re-write the formula for $\widehat{\Delta}(j_{1},...,j_{m})$. Denote:
\begin{equation}
\mathcal{Y}(d_{1},...,d_{m}) = ((\omega_{1}^{\mathrm{ort}})^{\top}z)^{d_{1}} \cdot ... \cdot ((\omega_{m}^{\mathrm{ort}})^{\top}z)^{d_{m}}    
\end{equation}
for $d_{1}=2j_{1},...,d_{m}=2j_{m}$.
Observe that $\mathcal{Y}(d_{1},...,d_{m})$ has the same distribution as $\mathcal{Y}^{\prime}(d_{1},...,d_{m})$ defined as:

\begin{equation}
\mathcal{Y}^{\prime}(d_{1},...,d_{m}) = (e_{1}^{\top}\frac{g}{\|g\|_{2}}\|z\|_{2})^{d_{1}} \cdot ... \cdot (e_{m}^{\top}\frac{g}{\|g\|_{2}}\|z\|_{2})^{d_{m}} \cdot
(\|\omega_{1}^{\mathrm{ort}}\|_{2})^{d_{1}} \cdot ... \cdot
(\|\omega_{m}^{\mathrm{ort}}\|_{2})^{d_{m}},
\end{equation}
where $g$ is a Gaussian vector taken from the $\mathcal{N}(0,I_{p})$ distribution, independently from: $\|\omega_{1}^{\mathrm{ort}}\|_{2},...,\|\omega_{m}^{\mathrm{ort}}\|_{2}$. 
This comes from the fact that for a fixed $z$ one can think about the set:
$\frac{\omega_{1}^{\mathrm{ort}}}{\|\omega_{1}^{\mathrm{ort}}\|_{2}},...,\frac{\omega_{m}^{\mathrm{ort}}}{\|\omega_{m}^{\mathrm{ort}}\|_{2}}$ as a random rotation of the system of $m$ canonical basis vectors: $e_{1},...,e_{m}$.
Thus instead of applying a random rotation to: $e_{1},...,e_{m}$, one can equivalently randomly rotate vector $z$. Randomly rotated vector $z$ has the same distribution as: $\frac{g}{\|g\|_{2}}\|z\|_{2}$. 

Now note that lengths of vectors $\omega_{1}^{\mathrm{ort}},...,\omega_{m}^{\mathrm{ort}}$ are chosen independently.

Therefore we obtain:
\begin{align}
\begin{split}
\mathbb{E}[((\omega_{1}^{\mathrm{ort}})^{\top}z)^{d_{1}} \cdot ... \cdot ((\omega_{m}^{\mathrm{ort}})^{\top}z)^{d_{m}}] = \\ \mathbb{E}[(\|\omega_{1}^{\mathrm{ort}}\|_{2})^{d_{1}}] \cdot ... \cdot \mathbb{E}[(\|\omega_{m}^{\mathrm{ort}}\|_{2})^{d_{m}}] \cdot \mathbb{E}[(e_{1}^{\top}v)^{d_{1}} \cdot ... \cdot (e_{m}^{\top}v)^{d_{m}}]
\|z\|_{2}^{d_{1}+...+d_{m}},
\end{split}
\end{align}
where $v \sim \frac{g}{\|g\|_{2}}$.

Denote $g=(g_{1},...,g_{p})^{\top}$.
Thus we obtain:
\begin{align}
\begin{split}
\label{lhs}
\mathbb{E}[((\omega_{1}^{\mathrm{ort}})^{\top}z)^{d_{1}} \cdot ... \cdot ((\omega_{m}^{\mathrm{ort}})^{\top}z)^{d_{m}}] = \\ \mathbb{E}[(\|\omega_{1}^{\mathrm{ort}}\|_{2})^{d_{1}}] \cdot ... \cdot \mathbb{E}[(\|\omega_{m}^{\mathrm{ort}}\|_{2})^{d_{m}}] \cdot 
\|z\|_{2}^{d_{1}+...+d_{m}} \mathbb{E}[\frac{g_{1}^{d_{1} \cdot ... \cdot}g_{m}^{d_{m}}}{\sqrt{g_{1}^{2}+...+g_{p}^{2}}^{d_{1}+...+d_{m}}}]
\end{split}
\end{align}

Now let us focus on the second expression from the formula on $\widehat{\Delta}(d_{1},...,d_{m})$. We have:
\begin{align}
\begin{split}
\label{rhs}
\mathbb{E}[((\omega_{1}^{\mathrm{iid}})^{\top}z)^{d_{1}} \cdot ... \cdot ((\omega_{m}^{\mathrm{iid}})^{\top}z)^{d_{m}}] = \prod_{i=1}^{m} \mathbb{E}[((\omega_{i}^{\mathrm{iid}})^{\top}z)^{d_{i}}]   
= \\ \mathbb{E}[(\|\omega_{1}^{\mathrm{iid}}\|_{2})^{d_{1}}] \cdot ... \cdot \mathbb{E}[(\|\omega_{m}^{\mathrm{iid}}\|_{2})^{d_{m}}] \cdot \|z\|_{2}^{d_{1}+...+d_{m}} \cdot \prod_{i=1}^{m} \mathbb{E}[\frac{g_{i}^{d_{i}}}{\sqrt{g_{1}^{2}+...+g_{p}^{2}}^{d_{i}}}],
\end{split}
\end{align}

where the first equality comes from the fact that
different $\omega_{i}^{\mathrm{iid}}$s are independent and the second one is implied by the analogous analysis to the one conducted above.

We will need the following lemma:

\begin{lemma}
\label{useful-lemma}
For every $s \in \mathbb{N}_{+}$ such that $s \leq n$ and every $k_{1},...,k_{s} \in \mathbb{N}_{+}$ the following holds:
\begin{equation}
\mathbb{E}[\frac{g_{1}^{k_{1}} \cdot ... \cdot g_{s}^{k_{s}}}{\sqrt{g_{1}^{2}+...+g_{p}^{2}}^{k_{1}+...+k_{s}}}] = \frac{\prod_{i=1}^{s}\mathbb{E}[g_{i}^{k_{i}}]}{\mathbb{E}[\sqrt{g_{1}^{2}+...+g_{p}^{2}}^{k_{1}+...+k_{s}}]}.    
\end{equation}
\end{lemma}

\begin{proof}
Take $r = \frac{g}{\|g\|_{2}}\|\tilde{g}\|_{2}$, where $\tilde{g}$ is an independent copy of $g$. Note that $r \sim g$.
We have:
\begin{align}
\begin{split}
\mathbb{E}[r_{1}^{k_{1}}] \cdot ... \cdot     
\mathbb{E}[r_{s}^{k_{s}}] = 
\mathbb{E}[r_{1}^{k_{1}} \cdot ... \cdot r_{s}^{k_{s}}]
= \mathbb{E}[\frac{g_{1}^{k_{1}} \cdot ... \cdot g_{s}^{k_{s}}}{\sqrt{g_{1}^{2}+...+g_{p}^{2}}^{k_{1}+...+k_{s}}}]
\cdot \mathbb{E}[\|\tilde{g}\|_{2}^{k_{1}+...+k_{s}}],
\end{split}    
\end{align}
where the first equality comes from the independence of different elements of $r=(r_{1},...,r_{n})^{\top}$
and the second equality is implied by the fact that $\tilde{g}$ is independent from $g$.

Therefore we have:
\begin{equation}
 \mathbb{E}[\frac{g_{1}^{k_{1}} \cdot ... \cdot g_{s}^{k_{s}}}{\sqrt{g_{1}^{2}+...+g_{p}^{2}}^{k_{1}+...+k_{s}}}] = \frac{\mathbb{E}[r_{1}^{k_{1}}] \cdot ... \cdot     \mathbb{E}[r_{s}^{k_{s}}]}{\mathbb{E}[\|\tilde{g}\|_{2}^{k_{1}+...+k_{s}}]}.   
\end{equation}
That completes the proof since $z \sim g$ and $\tilde{g} \sim g$.
\end{proof}

Note that by Lemma \ref{useful-lemma}, we can rewrite the right expression from the formula on 
$\widehat{\Delta}(d_1,..., d_m)$
as: 
\begin{equation}
\mathbb{E}[(\|\omega_{1}^{\mathrm{ort}}\|_{2})^{d_{1}}] \cdot ... \cdot \mathbb{E}[(\|\omega_{m}^{\mathrm{ort}}\|_{2})^{d_{m}}] \cdot \\
\|z\|_{2}^{d_{1}+...+d_{m}}\frac{\prod_{i=1}^{m}\mathbb{E}[g_{i}^{d_{i}}]}{\mathbb{E}[\sqrt{g_{1}^{2}+...+g_{p}^{2}}^{d_{1}+...+d_{m}}]}.
\end{equation}
The left expression from the formula on 
$\widehat{\Delta}(d_1,..., d_m)$ can be rewritten as:
\begin{align}
\begin{split}
\mathcal{L}(d_{1},...,d_{m}) = \mathbb{E}[(\|\omega_{1}^{\mathrm{iid}}\|_{2})^{d_{1}}] \cdot ... \cdot \mathbb{E}[(\|\omega_{m}^{\mathrm{iid}}\|_{2})^{d_{m}}] \cdot 
\|z\|_{2}^{d_{1}+...+d_{m}} \\
\frac{\prod_{i=1}^{m}\mathbb{E}[g_{i}^{d_{i}}]}
{\mathbb{E}[\sqrt{g_{1}^{2}+...+g_{p}^{2}}^{d_{1}}] \cdot ...\cdot \mathbb{E}[\sqrt{g_{1}^{2}+...+g_{p}^{2}}^{d_{m}}]}.
\end{split}
\end{align}

Since marginal distributions of $\omega_{i}^{\mathrm{ort}}$ and $\omega_{i}^{\mathrm{iid}}$ are the same, we can rewrite $\widehat{\Delta}(d_{1},...,d_{n})$ as:
\begin{equation}
\widehat{\Delta}(d_{1},...,d_{m})=
\mathcal{L}(d_{1},...,d_{m})(1 - \tau(d_{1},...,d_{m})),
\end{equation}
where $\tau(d_{1},...,d_{m})$ is defined as:
\begin{equation}
\tau(d_{1},...,d_{m}) = \frac{\mathbb{E}[\sqrt{g_{1}^{2}+...+g_{p}^{2}}^{d_{1}}] \cdot ...\cdot \mathbb{E}[\sqrt{g_{1}^{2}+...+g_{p}^{2}}^{d_{m}}]}
{\mathbb{E}[\sqrt{g_{1}^{2}+...+g_{p}^{2}}^{d_{1}+...+d_{m}}]}     
\end{equation}

With our new notation, $\Delta$ can be rewritten as:
\begin{align}
\begin{split}
\label{equation-x}
\Delta = \sum_{j=0}^{\infty} \frac{\theta^{j}}{j!} \sum_{(j_{1},...,j_{m}) \in \mathcal{S}_{j}} c(j_{1},...,j_{m})  \mathcal{L}(2j_{1},...,2j_{m})(1-\tau(2j_{1},...,2j_{m})),
\end{split}    
\end{align}

Note also that we have:
\begin{align}
\begin{split}
\label{equation-y}
e^{\theta(S_{1}^{\mathrm{iid}}+...+S_{m}^{\mathrm{iid}})} = \sum_{j=0}^{\infty} \frac{\theta^{j}}{j!} \sum_{(j_{1},...,j_{m}) \in \mathcal{S}_{j}} c(j_{1},...,j_{m})  \mathcal{L}(2j_{1},...2j_{m}).
\end{split}    
\end{align}

We need the following useful lemma:
\begin{lemma}
\label{tau-lemma}
The following holds if for some $i \neq j$ we have: $d_{i}, d_{j} > 0$ and all $d_{i}$ are even:
\begin{equation}
\tau(d_{1},...,d_{m}) \leq \frac{p}{p+2}.    
\end{equation}
\end{lemma}
\begin{proof}
Note that $\tau(d_{1},...,d_{m})$ can be rewritten as:
\begin{equation}
\label{multi-d}
\tau(d_{1},...,d_{m}) = \frac{\prod_{i=1}^{m} \mu_{p}(d_{i})}{\mu_{p}(\sum_{i=1}^{m} d_i)},    
\end{equation}
where $\mu_{p}(j)$ stands for the $j^{th}$ moment of the $\chi$-distribution with $p$ degrees of freedom.
Note that $\mu_{p}(j) = 2^{\frac{j}{2}}
\frac{\Gamma(\frac{p+j}{2})}{\Gamma(\frac{p}{2})}$,
where $\Gamma$ is the so-called \textit{Gamma-function}.

Using the fact that: $\Gamma(n) = (n-1)!$ and $\Gamma(n+\frac{1}{2})=\frac{(2n-1)!!}{2^{n}}\sqrt{\pi}$ for $n \in \mathbb{N}_{+}$, it is easy to see 
that for a fixed $p$, the RHS of the Equality \ref{multi-d} is maximized when $d_{i}=d_{j}=2$ and $d_{k}=0$ for some $i \neq j$ and $k \notin \{i,j\}$. Furthermore, straightforward calculations show that in that case the value of the RHS from Equality \ref{multi-d} is $\frac{p}{p+2}$. That completes the proof of the Lemma.
\end{proof}
By applying Eq. \ref{equation-x}, \ref{equation-y} and the above lemma, we conclude that for any $\lambda,\alpha>0$ the following is true:
\begin{equation}
\mathbb{E}[\exp(\frac{\lambda}{1-\alpha})\|R^{\mathrm{ort}}(x-y)\|_{2}^{2}] \leq \frac{p}{p+2}    
\mathbb{E}[\exp(\frac{\lambda}{1-\alpha})\|R(x-y)\|_{2}^{2}]
\end{equation}
Furthermore, from Corollary 1 in \cite{lin}, we get:
\begin{equation}
\mathbb{E}[\exp(-\frac{\lambda}{\alpha})\|R^{\mathrm{ort}}(x+y)\|_{2}^{2}] \leq     
\mathbb{E}[\exp(-\frac{\lambda}{\alpha})\|R(x+y)\|_{2}^{2}]
\end{equation}
Now observe that a fixed row of of $R^{\mathrm{ort}}$ is of the form:
\begin{equation}
(\sqrt{p}\frac{\sigma g_{1}}{\sqrt{g_{1}^{2}+...+g_{p}^{2}}},...,  
\sqrt{p}\frac{\sigma g_{p}}{\sqrt{g_{1}^{2}+...+g_{p}^{2}}}).
\end{equation}
From the fact that the regularized softmax kernel $\mathrm{SMREG}$ from \cite{performers} is upper-bounded by the softmax kernel $\mathrm{SM}$ (see: Theorem 1 from \cite{performers}), we get for $g \sim \mathcal{N}(0,1)$
and $g_{1},...,g_{p} \overset{\mathrm{iid}}{\mathrm{\sim}} \mathcal{N}(0,1)$:
\begin{equation}
\mathbb{E}[\mathrm{exp}(\sqrt{p}\frac{\sigma g_{1}}{\sqrt{g_{1}^{2}+...+g_{p}^{2}}})] \leq \mathbb{E}[\mathrm{exp}(\sigma g)].   
\end{equation}
Therefore entries of $R^{\mathrm{ort}}$ are sub-Gaussian with parameter $\sigma^{\prime} \leq \sigma$. Furthermore, from our previous analysis, we conclude that:
\begin{equation}
\mathbb{E}\left[\exp(\sum_{i=1}^{p}(\sqrt{p}\frac{\sigma  z_{i}}{\sqrt{g_{1}^{2}+...+g_{p}^{2}}})^{2})\right] \leq 
\prod_{i=1}^{p} \exp\left((\sqrt{p}\frac{\sigma  z_{i}}{\sqrt{g_{1}^{2}+...+g_{p}^{2}}})^{2}\right)
\end{equation}

Thus, we can mimick the analysis from the proof of Theorem 2.1 in \cite{dotpr}, but with strictly better upper bounds for the moment generating functions involved, and after standard algebraical transformations, obtain:
\begin{equation}
\mathbb{P}[(R^{\mathrm{ort}}x)^{\top}R^{\mathrm{ort}}y < m\sigma^{2}x^{\top}y-\epsilon m\sigma^{2}\|x\|_{2}\|y\|_{2}] < (1-\frac{2\rho(\epsilon)}{p+2})\exp(-\frac{m\epsilon^{2}}{8}).
\end{equation}

That proves Inequality \ref{lemma-ineq-1}.
The proof of Inequality \ref{lemma-ineq-2} is completely analogous, but this time $\rho(\epsilon)$ is replaced by $1-\rho(\epsilon)$.
\end{proof}

\subsection{Proof of Lemma \ref{lemma:1}} \label{app:lemma1proof}

\begin{proof}
For now, fix $1 \leq i, j_1, j_2 \leq L$, $j_1 \neq j_2$. We apply Theorem \ref{ortho-theorem} to $p = L$, $m = d/2$, $R^\mathrm{ort} = Y^\top$, $x = D_i, y = V_{j_1} - V_{j_2}$, $\sigma = \sqrt{2 / m}$. As the result, for any $0 < \epsilon < 1$ we have:
\begin{gather} 
    \mathbb{P} \left( \left| (X^{(1)}_i)^\top (X^{(2)}_{j_1} - X^{(2)}_{j_2}) - (D_i)^\top (V_{j_1} - V_{j_2} ) \right| > \epsilon \| D_i \|_2 \| V_{j_1} - V_{j_2} \|_2 \right) \\
    < 2 \left( 1 - \frac{1}{L + 2} \right) \exp\left( - \frac{\epsilon^2 d}{16} \right), \label{eq:dotpr1}
\end{gather}
where $\| \cdot \|_2$ denotes the Euclidean vector norm.

It is known that $\| B \|_2^2 = \sigma_1^2 = \rho (B B^\top)$, where $\rho (\cdot)$ denotes the spectral radius of the argument matrix. Next, it holds that
\begin{equation} \label{eq:rub}
    \rho (B B^\top) \leq \| B B^\top \|_1 \leq \| B \|_1 \| B^\top \|_1 = \max_{1 \leq j \leq L} \sum_{i = 1}^L | B_{i,j} | \cdot \max_{1 \leq i \leq L} \sum_{j = 1}^L | B_{i,j} |,
\end{equation}
where $\| \cdot \|_1$ is the induced 1-norm. The first transition in (\ref{eq:rub}) holds for spectral radius and induced norm of any matrix, while the second transition is due to submultiplicativity of the 1-norm.

Observe, that all elements of $B$ are nonnegative. Indeed,
\begin{equation*}
    \log A_{i,j} - \log A^{\min nz}_i - \log \epsilon_1 + \epsilon_2 \geq - \log \epsilon_1 + \epsilon_2 > 0,
\end{equation*}
where the first transition is due to $A_{i,j} \geq A^{\min nz}_i$, and the second is due to $\epsilon_1 < 1, \epsilon_2 > 0$.

In each row of $B$ there are up to $k$ nonzero elements. For each nonzero element $B_{i, j}$ it holds that
\begin{equation*}
    0 < B_{i, j} = \log (A_{i, j} / (A^{\min nz}_i \epsilon_1)) + \epsilon_2 \leq \log ( \gamma / \epsilon_1 ) + \epsilon_2 \leq \max (\log ( \gamma / \epsilon_1 ) + \epsilon_2, 1).
\end{equation*}
Therefore, we deduce that
\begin{equation} \label{eq:rub2}
    \max_{1 \leq i \leq L} \sum_{j = 1}^L | B_{i,j} | \leq k \max (\log ( \gamma / \epsilon_1 ) + \epsilon_2, 1) .
\end{equation}
Analogously it is shown that
\begin{equation} \label{eq:rub3}
    \max_{1 \leq j \leq L} \sum_{i = 1}^L | B_{i,j} | \leq k \max (\log ( \gamma / \epsilon_1 ) + \epsilon_2, 1) .
\end{equation}
Finally, we combine (\ref{eq:rub},\ref{eq:rub2},\ref{eq:rub3}) and deduce that
\begin{equation} \label{eq:sub}
    \sigma_L \leq \dots \leq \sigma_1 \leq k \max (\log ( \gamma / \epsilon_1 ) + \epsilon_2, 1) .
\end{equation}

We have:
\begin{equation*}
    \| U_i \|_2 = 1, \quad \| V_{j_1} - V_{j_2} \|_2 = \sqrt{\| V_{j_1} \|_2^2 + \| V_{j_2} \|_2^2 - 2 (V_{j_1})^\top V_{j_2}} = \sqrt{\| V_{j_1} \|_2^2 + \| V_{j_2} \|_2^2} = \sqrt{2}
\end{equation*}
due to orthogonality of $U, V$. Hence,
\begin{equation} \label{eq:rub4}
    \| D_i \|_2 \| V_{j_1} - V_{j_2} \|_2 = \sigma_i \| U_i \|_2 \| V_{j_1} - V_{j_2} \|_2 = \sqrt{2} \sigma_i \leq \sqrt{2} k \max (\log ( \gamma / \epsilon_1 ) + \epsilon_2, 1).
\end{equation}
Hence, we can replace $\| D_i \|_2 \| V_{j_1} - V_{j_2} \|_2$ in (\ref{eq:dotpr1}) by the right hand side from (\ref{eq:rub4}):
\begin{gather}
    \mathbb{P} \left( \left| (X^{(1)}_i)^\top (X^{(2)}_{j_1} - X^{(2)}_{j_2}) - (D_i)^\top (V_{j_1} - V_{j_2} ) \right| > \sqrt{2} \epsilon k \max (\log (\gamma / \epsilon_1) + \epsilon_2, 1) \right) \\
    < 2 \left( 1 - \frac{1}{L + 2} \right) \exp\left( - \frac{\epsilon^2 d}{16} \right). \label{eq:dotpr3}
\end{gather}

Our next step is to set
\begin{equation*}
    \epsilon = \frac{\epsilon_2}{\sqrt{2}} k^{-1} \max (\log (\gamma / \epsilon_1) + \epsilon_2, 1)^{-1}
\end{equation*}
and to write down a union bound for (\ref{eq:dotpr3}) over all $L^2 (L - 1)$ tuples of $(i, j_1, j_2)$ such that $1 \leq i, j_1, j_2 \leq L$, $j_1 \neq j_2$:
\begin{gather}
    \mathbb{P} \left( \exists i, j_1 \neq j_2 : \left| (X^{(1)}_i)^\top (X^{(2)}_{j_1} - X^{(2)}_{j_2}) - (D_i)^\top (V_{j_1} - V_{j_2} ) \right| > \frac{\epsilon_2}{2} \right) \\
    < \sum_{1 \leq i, j_1, j_2 \leq L, j_1 \neq j_2} \mathbb{P} \left( \left| (X^{(1)}_i)^\top (X^{(2)}_{j_1} - X^{(2)}_{j_2}) - (D_i)^\top (V_{j_1} - V_{j_2} ) \right| > \frac{\epsilon_2}{2} \right) \nonumber \\
    < 2 L^2 (L - 1) \left( 1 - \frac{1}{L + 2} \right) \exp\left( - \frac{1}{32} \epsilon_2^2 k^{-2} \max (\log (\gamma / \epsilon_1) + \epsilon_2, 1)^{-2} d \right) . \label{eq:pb}
\end{gather}

From the definition of $X, W_\mathcal{Q}, W_\mathcal{K}, D, V$ we know that $X W_\mathcal{Q} W_\mathcal{K}^\top X^\top = X^{(1)} (X^{(2)})^\top, D V^\top = B$. We combine (\ref{eq:dhidb},\ref{eq:pb}) and conclude the proof by observing that
\begin{align*}
    &\mathbb{P} \left( \forall 1 \leq i, j_1, j_2 \leq L, j_1 \neq j_2: \left| (X W_\mathcal{Q} W_\mathcal{K}^\top X^\top)_{i,j_1} - (X W_\mathcal{Q} W_\mathcal{K}^\top X^\top)_{i,j_2} - B_{i,j_1}  + B_{i,j_2} \right| < \epsilon_2 / 2 \right) \\
    &= 1 - \mathbb{P} \left( \exists 1 \leq i, j_1, j_2 \leq L, j_1 \neq j_2 : \left| (X^{(1)}_i)^\top (X^{(2)}_{j_1} - X^{(2)}_{j_2}) - (D_i)^\top (V_{j_1} - V_{j_2} ) \right| > \frac{\epsilon_2}{2} \right) \\
    &> 1 - 2 L^2 (L - 1) \left( 1 - \frac{1}{L + 2} \right) \exp\left( - \frac{1}{32} \epsilon_2^2 k^{-2} \max (\log (\gamma / \epsilon_1) + \epsilon_2, 1)^{-2} d \right) \\
    &\geq 1 - \left( 1 - \frac{1}{L + 2} \right) = \frac{1}{L + 2} .
\end{align*}
\end{proof}

\section{Additional experimental details and results} \label{app:expdet}

We use Tensorflow \cite{tensorflow} and a single NVIDIA P100 GPU for all experiments. To find empirical $d_{\min} (\epsilon_1, \epsilon_2)$, we gradually increase $d$ from $d_{lower}$ to $d_{upper}$ using a uniform grid of $30$ values:
\begin{itemize}
    \item When $k = 1$, we set $(d_{lower}, d_{upper}) = (200, 600)$.
    \item When $k = 2$, we set $(d_{lower}, d_{upper}) = (300, 1000)$.
\end{itemize}

Values of $d_{lower}, d_{upper}$ are selected so that on the whole range of $L$'s $d_{lower}$ is yet not enough to produce good enough $M$, while $d_{upper}$ is already enough.

Figure \ref{fig:amaps} illustrates other versions of Figure \ref{fig:amap} with different random seeds.

\begin{figure}
    \centering
    \includegraphics[width=\textwidth]{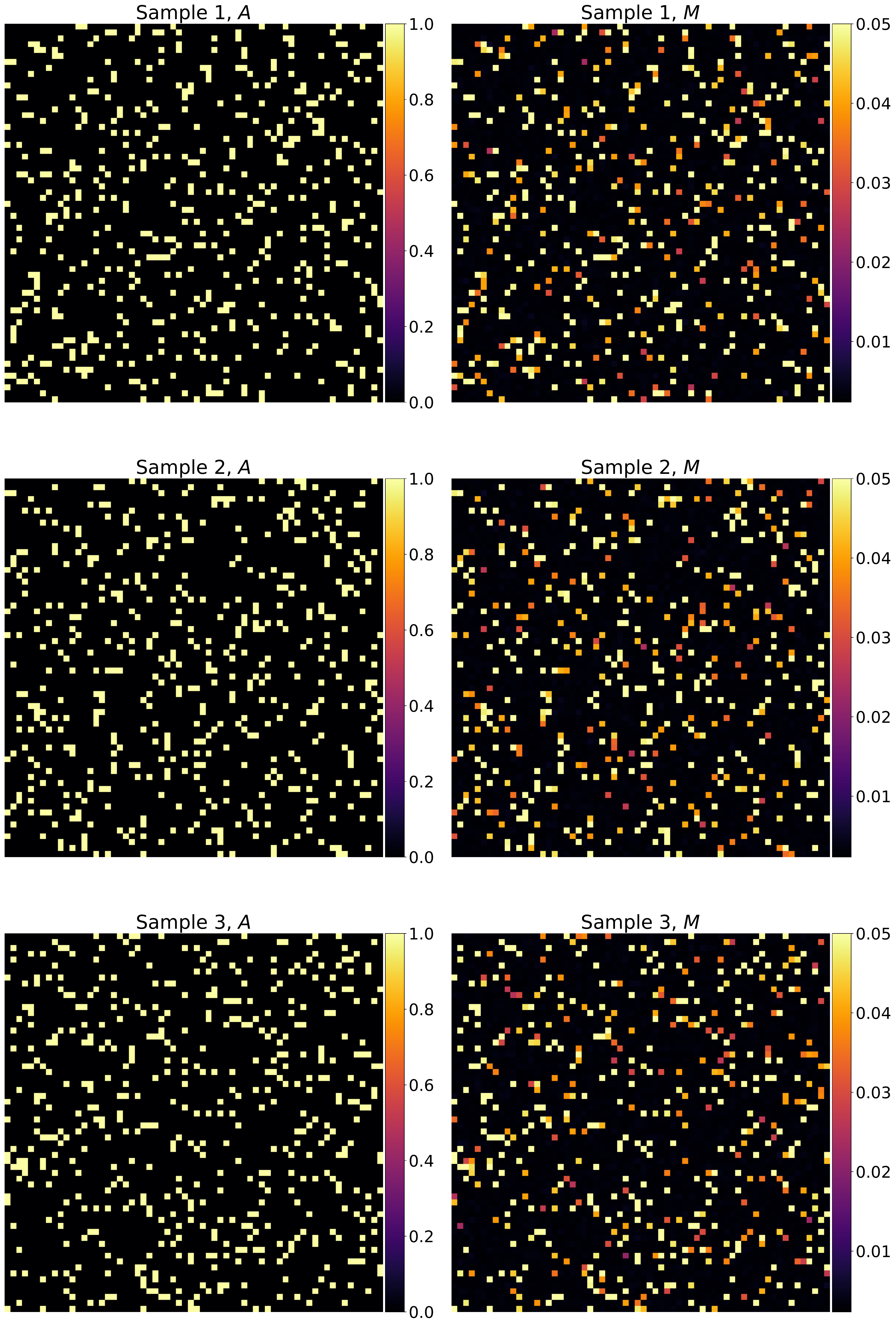}
    \caption{Other versions of Figure \ref{fig:amap} generated with different random seeds.}
    \label{fig:amaps}
\end{figure}

\end{document}